\documentclass{article} 
\usepackage{iclr2027_conference,times}

\usepackage{amsmath,amsfonts,bm}

\def\eqref#1{equation~\ref{#1}}

\def\1{\bm{1}}

\DeclareMathAlphabet{\mathsfit}{\encodingdefault}{\sfdefault}{m}{sl}
\SetMathAlphabet{\mathsfit}{bold}{\encodingdefault}{\sfdefault}{bx}{n}

\usepackage{hyperref}
\usepackage{url}
\usepackage{booktabs}
\usepackage{graphicx}
\usepackage{tikz}
\usepackage{multirow}
\usepackage{amsthm}

\newcommand{\METHOD}{S2MLVM}
\newcommand{\Mask}{\textsc{\METHOD-Mask}}
\newcommand{\Contrast}{\textsc{\METHOD-Contrast}}
\newcommand{\MaskShort}{\textsc{-Mask}}
\newcommand{\ContrastShort}{\textsc{-Contrast}}

\newcommand{\Sig}{\Sigma}

\newtheorem{remark}{Remark}
\newtheorem{lemma}{Lemma}

\title{Structured Latent Modeling for Supervised Multimodal Information Decomposition}

\iclrfinalcopy

\author{Wanting Huang \\
Department of Computer Science\\
University of Iowa\\
\texttt{wanting-huang@uiowa.edu} \\
\And
Sanvesh Srivastava \\
Department of Statistics \\
University of Iowa\\
\texttt{sanvesh-srivastava@uiowa.edu} \\
\AND
Weiran Wang \\
Department of Computer Science\\
University of Iowa\\
\texttt{weiran-wang@uiowa.edu}
}

\begin{document}

\maketitle

\begin{abstract}
Multimodal prediction relies on diverse forms of evidence: information repeated across modalities, cues specific to a single source, and complex cross-modal dependencies that emerge only when inputs are considered together. While recent methods promote richer interactions, they lack a principled way to isolate these target-relative contributions within learned continuous representations. We introduce a framework that applies contrastive or masked objectives at intermediate layers, coupled with source-wise invertible normalizing flows and a supervised, low-rank latent variable model. This architecture explicitly factorizes the joint distribution into shared task-relevant variation, modality-specific predictive variation, and task-irrelevant dependence. Drawing connections to prior multimodal learning assumptions, our approach evaluates how modalities independently and jointly contribute to the target. Ultimately, this framework unites intermediate representation learning with structured likelihood-based guidance, offering a practical latent-variable lens for characterizing continuous multimodal interactions. Empirically, we demonstrate the effectiveness of our approach across diverse multimodal benchmarks, showing robust improvements in predictive performance. 
\end{abstract}

\section{Introduction}
\label{sec:introduction}

Supervised multimodal fusion aims to learn representations that reflect the heterogeneity and interconnections between modalities \citep{liang2024multimodal}, capturing both overlapping (redundant) and modality-specific (unique) information to maximize predictive performance. Traditionally, multi-view learning relies on the redundancy assumption---presuming all task-relevant information is shared across modalities \citep{tosh2021contrastive,federici2020learning}. However, this often fails in real-world supervised settings where modalities carry unique predictive signals. When applied to complex multimodal data, standard late fusion architectures frequently suffer from unimodal bias (also known as modality competition or imbalance \citep{huang2022modality,zhang2024understanding}). During joint training, the fusion network tends to disproportionately rely on a dominant, easily accessible modality, inadvertently suppressing the learning of weaker modalities, and occasionally resulting in worse performance than their single-modality counterparts. A prominent example is visual question answering, where the vision modality is often ignored because the text already correlates strongly with the answer \citep{goyal2017making,cadene2019rubi}.
By analyzing the learning dynamics of deep fusion networks, recent studies~\cite{zhang2024understanding,kontras2025balancing} suggest that inter-modality correlation is the fundamental driver of this unimodal bias.

Recent approaches have attempted to address these challenges by quantifying or isolating multimodal interactions. Unsupervised representation learning methods \citep{liang2023factorized,dufumier2025what,wen2026infmasking} rely on handcrafted, label-preserving perturbations and contrastive learning to extract either a single representation per view, or a fused representation that captures both shared and unique predictive information. 
Meanwhile, supervised approaches use labels to filter out task-irrelevant information and mitigate modality competition through mechanisms such as gradient balancing, prototype-based modality rebalancing, and game-theoretic regularization \citep{wang2020what,Fan_2023_CVPR,kontras2025balancing}. However, they do not learn explicitly disentangled latent variables.
These developments motivate the important research question: \emph{how can a structured model of the source--target distribution support representation learning and mitigate unimodal bias?} Such a model must explicitly disentangle private and shared predictive information, accommodating nonlinear features while distinguishing source--target relationships from other cross-source associations.

To address these limitations, we propose a supervised latent variable model (LVM) framework centered on the \emph{Supervised Structured Multimodal Latent Variable Model} (\METHOD{}). This novel formulation generalizes both unsupervised \citep{BachJordan05a,lock2013joint} and existing supervised \citep{palzer2022sjive} LVMs. Its key design divides each view's latent space into distinct factors: a shared predictive component, a unique predictive component, and a shared nuisance (non-predictive) component.
Marginalizing these factors yields a low-rank-plus-diagonal covariance, imposing a compact dependence structure that accommodates shared source variation without requiring all of it to be directly predictive. Explicitly separating shared and private predictive signals from nuisance variations, our framework enables classifiers to optimally weight each component. Just as single-view disentanglement improves interpretability and sample efficiency \citep{bengio2013representation,locatello2020weakly}, this multimodal factorization directly empowers supervised fusion, boosting generalization and mitigating unimodal bias.
In summary, our main contributions are:
\begin{itemize}
    \item 
    We integrate the structured \METHOD{} covariance model into representation learning, using neural encoders as initial feature extractors coupled with dimension-preserving normalizing flows \citep{dinh2017density,papamakarios2021normalizing}. To accommodate high-dimensional, nonlinear multimodal data and optimize the statistical model jointly with the representations, we develop two instantiations (\Mask{} and \Contrast{}). These optimize a three-term objective combining a task loss, the joint Flow--\METHOD{} density criterion, and an auxiliary representation objective to maintain input information while reducing dimensionality. Both maintain modality-specific encoding without early cross-modal fusion: \Mask{} reconstructs masked clean features within each source, whereas \Contrast{} contrasts augmented multimodal observations.
    \item 
    We theoretically analyze our framework's connection to common assumptions in multimodal learning. By drawing connections to Partial Information Decomposition (PID,~\citealp{Williams2010Nonnegative,bertschinger2014quantifying}), we elucidate the mathematical mechanisms---such as cooperative nuisance suppression and V-structure exploitation---that generate synergistic information under our LVM.
    \item 
    We demonstrate the efficacy of \METHOD{} across synthetic and real-world benchmarks. In controlled settings, it accurately recovers task-relevant subspaces. Across diverse real-world datasets, both instantiations mitigate unimodal bias and outperform strong representation-learning baselines in two- and three-modality scenarios.
\end{itemize}

\section{Method}
\label{sec:method}

We develop \METHOD{}, a supervised multimodal learning framework that integrates representation learning with structured probabilistic modeling. Each modality is encoded via a source-specific pathway and transformed by an invertible normalizing flow. The resulting representations are jointly modeled with the target using a structured latent variable model that explicitly disentangles shared predictive, modality-specific predictive, and target-irrelevant variation. This joint model provides likelihood-based generative guidance during training. While the framework naturally extends to an arbitrary number of modalities (by incorporating additional source pathways and block-structured LVM components), we present the formulation for two sources $x^{(1)}$ and $x^{(2)}$ with target $y$ for clarity.

\subsection{Structured Latent Variable Modeling}
\label{sec:structured_latent_modeling}

Consider paired observations
$\mathcal D=\{(x_i^{(1)},x_i^{(2)},y_i)\}_{i=1}^{N}$.
Each source $x_i^{(m)}$ is mapped to a continuous representation through a source-specific encoder $g_{\theta_m}$ followed by an invertible flow  $f_{\phi_m}$:
\begin{equation}
    h^{(m)}=g_{\theta_m}(x^{(m)}) \in\mathbb R^{d_m},
    \qquad
    u^{(m)}=f_{\phi_m}(h^{(m)})\in\mathbb R^{d_m},
    \quad m\in\{1,2\},
    \label{eq:method_source_maps}
\end{equation}
and these representations do not depend on the target $y$.
We also extract continuous embedding $\tilde{y} \in\mathbb R^{d_y}$ of the target $y$ if it is categorical. 

Our framework, the Supervised Structured Multimodal Latent Variable Model (\METHOD{}), defines a joint generative process over the concatenated features $w=[u^{(1)};u^{(2)};\tilde y] \in \mathbb R^{d}$, where $d=d_1+d_2+d_y$. We model $w$ using four independent latent factors $z=[z^{12};z^1;z^2;z^c]$ of dimension $K=k_{12}+k_1+k_2+k_c$. The generative model $w = Az + \epsilon$ follows the structured linear mapping:
\begin{equation}
    \begin{bmatrix}
        u^{(1)} \\
        u^{(2)} \\
        \tilde y
    \end{bmatrix}
    =
    \begin{bmatrix}
        \Lambda_1^{12} & \Lambda_1^1 & 0 & \Lambda_1^c\\
        \Lambda_2^{12} & 0 & \Lambda_2^2 & \Lambda_2^c\\
        B^{12} & B^1 & B^2 & 0
    \end{bmatrix}
    \begin{bmatrix}
        z^{12} \\
        z^1 \\
        z^2 \\
        z^c
    \end{bmatrix}
    +
    \begin{bmatrix}
        \epsilon_1 \\
        \epsilon_2 \\
        \epsilon_y
    \end{bmatrix}, \qquad z\sim\mathcal N(0,I_K),
    \label{eq:method_fssfa_loadings}
\end{equation}
where $\Lambda_m^{12}$, $\Lambda_m^{m}$, and $\Lambda_m^c$ represent the shared predictive, modality-specific predictive, and target-irrelevant components, respectively, for modality $m$. Loading matrices $B^{12}$ and $B^m$ characterize the effects of shared and modality $m$-specific latent variables on $\tilde y$. The observation noise $\epsilon = [\epsilon_1; \epsilon_2; \epsilon_y]$ is drawn from $\mathcal N(0,\Psi)$, with a positive diagonal covariance $\Psi=\operatorname{diag}(\Psi_1,\Psi_2,\Psi_y)$.

This block-structured loading pattern explicitly disentangles the latent space into distinct semantic components. The factor $z^{12}$ captures shared predictive information that influences both sources and the target, while $z^1$ and $z^2$ isolate unique predictive information specific to each source. Crucially, the target $\tilde y$ does not depend on the shared nuisance factor $z^c$. This structure ensures that task-irrelevant cross-modal correlations are absorbed by $z^c$ rather than confounding the predictive representations. Marginalizing over the latent factors gives
\begin{equation}
    w\sim\mathcal N(0,\Omega),
    \qquad
    \Omega=AA^\top+\Psi.
    \label{eq:method_fssfa_covariance}
\end{equation}
This formulation provides a unified perspective on multi-view latent variable models. In the unsupervised setting (where the target $y$ is absent), it is impossible to distinguish predictive shared variation ($z^{12}$) from nuisance shared variation ($z^c$), so they collapse into a single shared factor. Under this restriction, the model reduces to probabilistic Canonical Correlation Analysis (CCA) \citep{BachJordan05a} if we retain only this shared factor, and to Joint and Individual Variation Explained (JIVE) \citep{lock2013joint} if we also include the unique factors $z^1$ and $z^2$. Conversely, the full supervised structure generalizes recent predictive LVMs like supervised JIVE \citep{palzer2022sjive}, which lack the nuisance factor $z^c$. By introducing $z^c$, \METHOD{} explicitly separates shared predictive variation from task-irrelevant cross-modal correlations. 

\subsection{Flow-Based Distribution Modeling}
\label{sec:flow_distribution_modeling}

To apply the Gaussian \METHOD{} to complex multimodal data, we must bridge the gap between the structured Gaussian base distribution and the highly non-Gaussian neural representations $h^{(m)}$. We achieve this using source-wise normalizing flows $f_{\phi_m}$, which provide invertible, dimension-preserving transformations with tractable Jacobian determinants \citep{dinh2017density,papamakarios2021normalizing}. Although each flow processes a single modality independently, their outputs $u^{(m)}$ are modeled jointly with the target $\tilde y$ under the \METHOD{} base distribution to retain all source--source and source--target dependencies. This unified construction enables a joint maximum likelihood framework: we can perform maximum likelihood estimation of the structured loading matrix $A$ and residual covariance $\Psi$ while simultaneously learning the flow parameters $\phi = (\phi_1, \phi_2)$ to map the initial encoder outputs into Gaussian-distributed variables.

For fixed initial encoders, the joint log-density of this feature-space model is given by
\begin{equation}
\begin{aligned}
    \log p_{\phi,A,\Psi}(h^{(1)},h^{(2)},\tilde y)
    ={}&\log\mathcal N(w;0,\Omega) +\sum_{m=1}^{2}
       \log\left|\det J_{f_{\phi_m}}(h^{(m)})\right|,
\end{aligned}
\label{eq:method_feature_density}
\end{equation}
where $\Omega = AA^\top + \Psi$. The corresponding fitting objective for an observation is the negative log-likelihood, $\mathcal L_{\mathrm{density}} = -\log p_{\phi,A,\Psi}(h^{(1)},h^{(2)},\tilde y)$. The Gaussian term fits the structured joint relationships, while the Jacobian terms account for volume changes induced by the flows. 

While finite-capacity flows practically optimize the Kullback-Leibler divergence between the empirical distribution and the Gaussian LVM, normalizing flows theoretically possess universal approximation capabilities for continuous probability distributions \citep{papamakarios2021normalizing}. Because our flows operate independently on each modality, this provides a strong asymptotic guarantee that, given sufficient capacity, they can perfectly transform the source representations to match the Gaussian marginals required by the joint LVM.

\subsection{Auxiliary Representation Learning}
\label{sec:auxiliary_representation_learning}

Because the initial encoders perform dimension reduction on the raw inputs, an auxiliary objective is necessary to prevent representation degeneracy before the representations enter the structured LVM. Although the subsequent normalizing flows are invertible, this auxiliary loss ensures the encoders extract and retain sufficient information. The objective $\mathcal L_{\mathrm{rep}}$ is computed directly on the encoder representations $h^{(m)}$ and is instantiated by either masked reconstruction ($\mathcal L_{\mathrm{rec}}$) or contrastive learning ($\mathcal L_{\mathrm{con}}$). Only one of these alternatives is used in each variant; they are not added together.

\paragraph{Masked reconstruction.}
Following recent masked representation learning methods \citep{he2022masked,baevski2022data2vec}, let $H^{(m)}$ denote clean features at an intermediate layer, $\widehat H^{(m)}$ the reconstructed features, and $M_m$ the set of masked indices for modality $m$. We define
\begin{equation}
    \mathcal L_{\mathrm{rec}}
    =\frac{1}{2}\sum_{m=1}^{2}
    \ell\!\left(
        \widehat H^{(m)}[M_m],
        \operatorname{sg}(H^{(m)}[M_m])
    \right),
    \label{eq:method_reconstruction_loss}
\end{equation}
where $\ell$ is a distance metric (e.g., Smooth L1 loss) averaged over the masked entries, and $\operatorname{sg}$ stops gradients through the target argument. The objective encourages recovery of source features from partial observations.

\paragraph{Contrastive learning.}
Let $v$ and $v^+$ be normalized projections of the representations $h^{(m)}$ extracted from two independently augmented versions of the same modality. The InfoNCE objective \citep{oord2018representation,chen2020simple} for this observation is:
\begin{equation}
    \mathcal L_{\mathrm{con}}
    =-\log\frac{\exp(v^\top v^+ / \tau)}{\exp(v^\top v^+ / \tau) + \sum_{v^- \in \mathcal N(v)}\exp(v^\top v^- / \tau)},
    \label{eq:method_contrastive_loss}
\end{equation}
where $\tau>0$ is the temperature, and $\mathcal N(v)$ contains negative samples from other observations. In practice, this loss is symmetrized across both anchor directions. Crucially, positive pairs are formed from the same joint observation rather than across modalities, leaving the extraction of cross-modal shared information to the \METHOD{}.

\subsection{Total Representation Learning Objective}
\label{sec:multimodal_representation_learning}

Because our framework explicitly isolates predictive components, we construct the final representation by computing and concatenating the posterior means of the predictive latent factors ($z^{12}, z^1$, and $z^2$) given $(u^{(1)}, u^{(2)})$ from the \METHOD{}.
To provide direct discriminative supervision and accommodate standard downstream evaluations, this representation is passed to a multilayer perceptron (MLP) optimized with a standard task loss $\mathcal L_{\mathrm{MLP}}$ (e.g., cross-entropy for classification or mean squared error for regression) against the ground-truth targets.

The overall training objective minimizes a weighted combination of this discriminative task loss, the generative Flow--\METHOD{} maximum-likelihood loss ($\mathcal L_{\mathrm{density}}$) defined in Section~\ref{sec:flow_distribution_modeling}, and the auxiliary representation loss ($\mathcal L_{\mathrm{rep}}$) defined in Section~\ref{sec:auxiliary_representation_learning}:
\begin{equation}
    \mathcal L
    =\mathcal L_{\mathrm{MLP}}
     +\lambda_{\mathrm{density}}\mathcal L_{\mathrm{density}}
     +\lambda_{\mathrm{rep}}\mathcal L_{\mathrm{rep}}.
    \label{eq:method_total_loss}
\end{equation}

\paragraph{Optimization.}
In practice, the total loss $\mathcal L$ is averaged over the training set and minimized using stochastic gradient descent (SGD). All parameters---including source encoders, normalizing flows, structured covariance parameters, and the final MLP classifier---are optimized jointly end-to-end.

\section{Connections to other multimodal learning frameworks}
\label{sec:connection}

To understand how our model captures complex multimodal interactions, we analyze it through the lens of Partial Information Decomposition (PID)~\citep{Williams2010Nonnegative,liang2024foundations}, a framework that decomposes joint predictive information into unique, redundant, and synergistic components. Our \METHOD{} connects deeply to this formalism. PID achieves this decomposition by defining a constraint space $\Delta_P$ of adversarial joint distributions $Q$ that preserve the true source-target marginals:
\begin{equation}
    \Delta_P = \{ Q \mid Q(u^{(1)}, \tilde y) = P(u^{(1)}, \tilde y), \ Q(u^{(2)}, \tilde y) = P(u^{(2)}, \tilde y) \}.
\end{equation}
While the latent factors in our model ($z^{12}, z^1, z^2, z^c$) possess clear intuitive meanings, they do not map one-to-one onto the information-theoretic atoms of PID. Instead, as we detail in Appendix~\ref{sec:interpretation}, the formal PID definition extracts synergy by identifying a worst-case adversary $q^*_{MI}$ that minimizes the joint mutual information:
\begin{equation}
    q^*_{MI} = \arg\min_{Q \in \Delta_P} I_Q(u^{(1)}, u^{(2)}; \tilde y).
\end{equation}
We show that this adversary is remarkably strong: it correlates the private noise across modalities to align with the combined task signal, maximally reducing the joint predictive information.

Furthermore, we demonstrate that the assumption of conditional independence ($u^{(1)} \perp\!\!\!\perp u^{(2)} \mid \tilde y$)---a widespread premise in multimodal machine learning~\citep{blum1998combining,Chaudh_09a}---provides a tractable analytic proxy $q^*_{CE}$ for the intractable PID optimization. This adversary is defined by maximizing the conditional entropy:
\begin{equation}
    q^*_{CE} = \arg\max_{Q \in \Delta_P} H_Q(u^{(1)}, u^{(2)} \mid \tilde y).
\end{equation}
Because our generative structure models the latent features as jointly Gaussian, their information quantities are fully determined by the joint covariance matrix. Consequently, the constraint $Q \in \Delta_P$, which fixes the source-target marginals, restricts the adversary to exclusively manipulating the cross-source covariance block. This means the adversarial covariance matrix $\Sig^{(Q)}$ is identically the original covariance matrix $\Sig$ plus an off-diagonal cross-source perturbation block $C$. As detailed in Appendix~\ref{sec:interpretation}, the optimal conditional entropy adversary acts as a decoupling filter by injecting the exact negative conditional cross-covariance $C_{CE} = -\Sig_{u^{(1)} u^{(2)} \mid \tilde y}$ as this perturbation. Under our \METHOD{} parameterization, this injected covariance takes a closed form:
\begin{equation}
    C_{CE} = -\Big( \underbrace{\Lambda_1^c (\Lambda_2^c)^\top}_{\text{shared nuisance}} + \underbrace{\Lambda_1^{12} (\Lambda_2^{12})^\top}_{\text{shared signal}} - \underbrace{\Sig_{u^{(1)} \tilde y} \Sig_{\tilde y \tilde y}^{-1} \Sig_{\tilde y u^{(2)}}}_{\text{induced V-structure}} \Big).
\end{equation}
This term-by-term algebraic decomposition reveals that \emph{all} latent factors provide opportunities for synergistic predictive power. To eliminate synergy and achieve conditional independence, the adversary $C_{CE}$ must systematically inject noise to cancel three distinct mechanisms: the cooperative suppression of the shared nuisance $z^c$, the noise-averaging of the redundant shared signal $z^{12}$, and the V-structure ``explaining-away'' correlation induced by the independent private signals $z^1$ and $z^2$. 

Because $q^*_{CE}$ provides a closed-form but sub-optimal minimization of the joint mutual information ($I_{q^*_{MI}} \le I_{q^*_{CE}}$), substituting it into the PID equations yields analytic bounds on the true optimal quantities. Specifically, our tractable proxy safely underestimates both Synergy ($S_{CE} \le S_{MI}$) and Redundancy ($R_{CE} \le R_{MI}$), while overestimating the Unique Information ($U_j^{CE} \ge U_j^{MI}$). 
Finally, our structured model explicitly accommodates the spectrum of multimodal assumptions. Many classical multi-view frameworks rely on the \emph{multi-view redundancy assumption}~\citep{federici2020learning,tsai2021selfsupervised,tosh2021contrastive}, positing that all task-relevant information is shared ($z^1, z^2$ are empty). By explicitly isolating the private predictive factors ($z^1, z^2$) from the shared redundancy ($z^{12}$), \METHOD{} provides a generalization that efficiently learns in both highly redundant and highly unique multimodal settings without discarding synergistic interactions.

\section{Related Work}
\label{sec:related_work}
\paragraph{Unsupervised multimodal representation learning.}
Multimodal representation learning traditionally relies on correspondence and partial observations to retain information across sources. Contrastive learning and mutual information-based methods~\citep{oord2018representation,belghazi2018mutual,chen2020simple,tian2020contrastive,radford2021learning} learn transferable representations by aligning paired views. The Multi-View Information Bottleneck~\citep{federici2020learning} formalizes this by providing an information-theoretic account for retaining shared predictive content under strict multi-view redundancy assumptions.
To move beyond purely shared information, recent works explicitly isolate modality-specific or complementary content. Deep generative models~\citep{wang2016vcca,lee2021private,palumbo2023mmvaeplus,zhang2026disentanglement} separate shared and private variations through structured likelihood modeling. Parallel efforts in self-supervised learning, such as FactorCL~\citep{liang2023factorized} and DisentangledSSL~\citep{wang2025information}, achieve this factorization using complex augmentations and contrastive learning. Other approaches actively encourage complementary interactions: CoMM~\citep{dufumier2025what} models joint spaces to capture beyond-redundancy interactions, while InfMasking~\citep{wen2026infmasking} contrasts partially masked features against complete fusions to elicit synergy. While these self-supervised methods demonstrate the necessity of learning from more than just cross-modal agreement, they largely rely on complex ad-hoc regularizations or masking strategies.
We generalize this principle of likelihood-based separation to the supervised setting, providing a rigorous statistical formalization for complex source--target dependencies.

\paragraph{Balancing supervised multimodal optimization.}
In joint multimodal training, models often disproportionately rely on a single dominant modality~\citep{peng2022balanced, huang2022modality}, a phenomenon formally linked to inter-modality correlations~\citep{zhang2024understanding}. To mitigate this bias, a prominent line of work directly intervenes in the optimization process. Techniques like OGM~\citep{peng2022balanced}, AGM~\citep{li2023boosting}, and MLB~\citep{kontras2024improving} dynamically scale gradient updates based on ongoing modality contributions, while scheduling methods like MMPareto~\citep{wei2024mmpareto}, MLA~\citep{zhang2024multimodal}, and ReconBoost~\citep{hua2024reconboost} explicitly reconcile conflicting unimodal and multimodal objectives. 
More recently, methods use game-theoretic valuation~\citep{wei2024enhancing} or mutual-information decomposition~\citep{kontras2025balancing} to formalize modality cooperation. 
We share the motivation of using task-related dependence to guide learning; however, \METHOD{} achieves this balance inherently via its generative objective, bypassing the need for explicit gradient modulation.

\paragraph{Target-relative information decomposition.}
Partial information decomposition (PID) isolates the unique, redundant, and synergistic information provided by multiple sources about a target~\citep{Williams2010Nonnegative,bertschinger2014quantifying}. Recent works have extended these concepts to multimodal learning, developing scalable estimators for quantifying interactions~\citep{liang2023quantifying} and deriving learning guarantee
~\citep{liang2024multimodal}. While these works focus on estimating information-theoretic quantities from datasets or fixed representations, evaluating PID on continuous, high-dimensional neural representations presents significant computational challenges. Existing continuous estimators rely on variational optimization~\citep{pakman2021estimating}, neural estimation~\citep{kleinman2021redundant}, convex approximations for Gaussians~\citep{venkatesh2022partial,venkatesh2023gaussian}, or invertible transformations to latent Gaussian spaces~\citep{zhao2025partial}. We build on this latent-Gaussian principle rather than proposing a new PID definition. 
By explicitly disentangling shared and modality-specific predictive factors, our framework naturally yields the covariance structures needed to analyze synergistic information.

\section{Experiments}
\label{sec:experiments}

We adopt the experimental setups, including datasets and base architectures, of recent related works. For all experiments, we tune hyperparameters and perform ablation studies exclusively on the validation set. We select the best-performing model configuration on the validation set and evaluate its task MLP predictions to report the final metrics on the held-out test set. We report accuracy for classification tasks and mean squared error (MSE) for regression tasks. 
Our repeated-run results use 5 seeds 
and are summarized by the mean and sample standard deviation.
We focus on comparing with the current state-of-the-art unsupervised (e.g., COMM~\citep{dufumier2025what}, InfMasking~\citep{wen2026infmasking}) and supervised (MCR~\citep{kontras2025balancing}) methods using their configurations (data, architecture, and implementation); results taken from these prior works are annotated with $^*$.

\subsection{Subspace Recovery on Synthetic Data}
\label{sec:synthetic_subspace_recovery}

We first verify that maximum likelihood modeling via SGD recovers the true task-relevant latent structure. We generate two observed sources and the response from independent Gaussian latent blocks ($z^{12}$, $z^1$, $z^2$, and $z^c$) following the \METHOD{} loading pattern in~\eqref{eq:method_fssfa_loadings}. 
We use isotropic observation noise with covariance $\Psi=\sigma^2 I$.
we then apply fixed random nonlinear flow transformations to both observed sources (without the nonlinear mappings, the recovery would be perfect across settings). 
For the experiments, we first vary the ambient source dimension over $d_1=d_2\in\{12,24,48,96\}$ while fixing $k=k_{12}=k_1=k_2=k_c=4$ and $\sigma=0.4$. Second, we vary the observation-noise standard deviation over $\sigma\in\{0.2,0.4,0.6,0.8\}$ while fixing $d_1=d_2=24$ and $k=4$. Then, we vary the dimensionality of each latent block over $k\in\{4,8,12,16\}$ while fixing $d_1=d_2=48$ and $\sigma=0.4$. 
Models are trained on $10{,}000$ samples for $12{,}000$ steps using SGD (momentum $0.9$, batch size $512$, learning rate $3\times10^{-3}$). The checkpoint achieving the lowest negative log-likelihood on a $2{,}000$-sample validation set is evaluated on $4{,}096$ held-out test samples. Results are averaged over five random seeds.

\begin{table}[t]
    \centering
    \caption{
        Recovery of \METHOD{} subspaces. Here $dim=d_1=d_2$ and $k=k_{12}=k_1=k_2=k_c$.}
    \label{tab:synthetic_recovery}
    \begin{tabular}{@{}llc@{}}
        \toprule
        \multicolumn{2}{c}{Setting}
        & Canonical correlation (CC)  $\uparrow$ \\
        \midrule

        \multirow{4}{*}{$k=4$, $\sigma=0.4$}
        & $dim=12$ & $0.99479\pm0.00315$ \\
        & $dim=24$ & $0.98816\pm0.00396$ \\
        & $dim=48$ & $0.94867\pm0.03073$ \\
        & $dim=96$ & $0.93104\pm0.03477$ \\
        \hline 
        
        \multirow{4}{*}{$dim=24$, $k=4$}
        & $\sigma=0.2$ & $0.98624\pm0.00523$ \\
        & $\sigma=0.4$ & $0.98816\pm0.00396$ \\
        & $\sigma=0.6$ & $0.99109\pm0.00206$ \\
        & $\sigma=0.8$ & $0.99137\pm0.00155$ \\
        \hline

        $dim=48$, $\sigma=0.4$
        &
        $k=4$
        & $0.94867 \pm 0.03073$ \\
        &
        $k=8$
        & $0.96488 \pm 0.01118$ \\
        &
        $k=12$
        & $0.97333 \pm 0.00559$ \\
        &
        $k=16$
        & $0.96502 \pm 0.00804$ \\

        \bottomrule
    \end{tabular}
\end{table}

Because individual loading coordinates are only identifiable up to rotations within each latent block, we evaluate recovery at the subspace level using a rotation-invariant metric:
\begin{equation}
\mathrm{CC} = \frac{1}{k_{12}+k_{1}+k_{2}}\sum_{i=1}^{k_{12}+k_{1}+k_{2}} \sigma_i\!\left(Q^\top\widehat{Q}\right).
\end{equation}
Here, $Q$ and $\widehat{Q}$ denote orthonormal bases for the true and estimated task-relevant subspaces, respectively, while $\sigma_i$ indicates the $i$-th singular value. Mean canonical correlation ($\mathrm{CC}$) measures subspace alignment through their principal angles (larger is better), and is bounded in $[0,1]$.
Table~\ref{tab:synthetic_recovery} demonstrates that \METHOD{} consistently recovers the task-relevant subspace under varying conditions. While performance degrades smoothly as the estimation problem becomes more challenging---such as with higher ambient or latent dimensions---the recovered predictive geometry remains robust, maintaining high canonical correlation across all settings.

\subsection{Experiments on TriFeature}
\label{sec:trifeature}

We use TriFeatures~\citep{dufumier2025what} as a controlled diagnostic to examine whether the latent variables learned by our model exhibit the intended functional specialization. By providing explicit control over target labels, TriFeatures enables us to systematically evaluate the extraction of shared, source-specific, and cross-source information.  Each input image comprises three categorical attributes—shape, texture, and color—which are used to construct four targeted diagnostic tasks. Redundancy ($R$) targets the shape attribute shared across both images. The two unique tasks, $U_{1}$ and $U_{2}$, correspond to the independent texture attributes of the first and second images, respectively. These three attribute-level tasks are evaluated as 10-way classification problems. Finally, Synergy ($S$) is a binary classification task evaluating a cross-source relation determined jointly by the texture of the first image and the color of the second. Crucially, neither attribute alone is sufficient to predict the target; the synergistic relation can only be resolved by integrating information contributed by both sources.
Following InfMasking \citep{wen2026infmasking}, we train on $10{,}000$ synthetic image pairs and evaluate on $4{,}096$ held-out pairs.
For the encoders of both \Mask{} and \Contrast{}, each view uses an AlexNet frontend and an independent width-512 Transformer with its own CLS token. We train both architectures for $100$ epochs. The representation dimension is $512$ for $u^{(1)}$, $u^{(2)}$, and $\tilde{y}$. The latent dimensions are $k_{12}=k_{1}=k_{2}=k_c=8$.

\begin{table}[t]
    \centering
    \caption{Classification accuracy (\%) on TriFeature, for different task labels $R$/$U_1$/$U_2$/$S$.}
    \label{tab:trifeature_main}
    \begin{tabular}{@{}lccccc@{}}
        \toprule
Method & $R$-ACC $\uparrow$ & $U_1$-ACC $\uparrow$ & $U_2$-ACC $\uparrow$ & $U$-ACC $\uparrow$ & $S$-ACC $\uparrow$ \\
        \midrule
FactorCL 
    & $99.8^{*}$ &  & & $62.5^{*}$ & $46.5^{*}$ \\
CoMM 
    & $99.9\pm 0.1^{*}$ & $84.4 \pm 2.4^{*}$ & $91.2 \pm 1.0^{*}$ & $86.8\pm 3.0^{*}$  & $71.4\pm 3.5^{*}$ \\
InfMasking 
    & $99.9\pm 0.1^{*}$ & $90.7\pm 2.1^{*}$ & $91.4\pm 3.0^{*}$ & $90.6\pm 2.3^{*}$& $77.0\pm 4.2^{*}$ \\
MCR 
& $99.9\pm0.1$ & $92.1\pm1.4$ & $93.2\pm0.9$ & $92.7\pm1.2$ & $90.2\pm3.9$ \\
\midrule
\Mask{}
    & $93.4 \pm 1.1$
    & $91.3 \pm 2.2$
    & $84.9 \pm 3.5$
    & $88.1 \pm 2.7$
    & $85.4 \pm 4.3$ \\
\Contrast{}
    & $\mathbf{99.9} \pm \mathbf{0.1}$
    & $\mathbf{95.1} \pm \mathbf{0.7}$
    & $\mathbf{95.3} \pm \mathbf{0.8}$
    & $\mathbf{95.2} \pm \mathbf{0.5}$
    & $\mathbf{92.3} \pm \mathbf{0.9}$ \\
    $\quad  z^{12}$
            & $93.6 \pm 3.3$ & $34.2 \pm 7.1$ & $32.7 \pm 11.8$ & $33.5 \pm9.5$  & $56.2 \pm2.6$  \\
    $\quad  z^1$
            & $70.4 \pm 4.3$ & $83.5 \pm 9.2$ & $14.1 \pm 7.5$ & $48.8 \pm 8.4$ & $48.6 \pm 1.6$  \\
    $\quad  z^2$
            & $64.3 \pm 9.5$ & $10.6 \pm 5.2$ & $78.4 \pm 14.0$  & $44.5 \pm9.6$ & $49.8 \pm 1.7$  \\
\bottomrule
\end{tabular}
\end{table}

We provide the classification accuracy on test set of several methods in Table~\ref{tab:trifeature_main}. In general, supervised methods (MCR and \METHOD{}) outperform unsupervised methods (e.g., COMM and InfMasking), due to joint training of predictors and predictive representations, and contrastive learning turns out to be a better representation learning objective for this dataset.
To see that our latent factors well extracts the desired predictive information, after model training (with task loss on $(z^{12}$, $z^{1}, z^{2})$), we train a small MLP on each extracted latent factor by \Contrast{} and obtain 3 additional accuracies. As expected, For the tasks $R$, $U_1$ and $U_2$, performing classification only on $z^{12}$, $z^{1}, z^{2}$ correspondingly maintains most accuracy. For the synergy task $S$ however, each individual factor is not sufficient to maintain desired predictive information. 

\subsection{Real-World Multimodal Benchmarks}
\label{sec:multibench}

\paragraph{Datasets and tasks.}
We evaluate two groups of real-world benchmarks. The first follows the InfMasking setup~\citep{wen2026infmasking} on MultiBench~\citep{liang2021multibench}, covering Vision\&Touch end-effector (V\&T EE) prediction, MIMIC diagnostic-group prediction, MOSI sentiment classification, UR-FUNNY humor recognition, and MUSTARD sarcasm recognition. The second follows the MCR configurations for CREMA-D, UCF101, and MOSEI~\citep{kontras2025balancing}. Complete dataset definitions and evaluation conventions are provided in Appendix~\ref{sec:dataset}; dataset-specific model architetures are summarized in Appendix~\ref{sec:expts-architecture}.

\begin{table*}[t]
    \centering
    \caption{
    Test results on two modality benchmarks, by unsupervised, supervised, and our methods.}
    \label{tab:realworld_unified}
    \begin{tabular}{@{}l@{\hspace{0.02\linewidth}}c@{\hspace{0.02\linewidth}}c@{\hspace{0.02\linewidth}}c@{\hspace{0.02\linewidth}}c@{\hspace{0.02\linewidth}}c@{\hspace{0.02\linewidth}}c@{}}
        \toprule

        Method
        & V\&T EE $\downarrow$
        & MIMIC $\uparrow$
        & CREMA-D $\uparrow$
        & UCF101 $\uparrow$
        & MOSEI $\uparrow$
        \\
        \midrule

        FactorCL 
        & $10.8\pm0.6^{*}$
        & $67.3\pm0.0^{*}$
        & $61.1 \pm1.3$
        & $50.3 \pm2.2$
        & $78.5 \pm 0.1$
        \\

        CoMM 
        & $8.0\pm2.1^{*}$
        & $66.4\pm0.4^{*}$
        & $59.9 \pm 2.5$
        & $51.0 \pm 2.1$
        & $79.7 \pm 0.3$
        \\

        InfMasking 
        & $4.2\pm0.4^{*}$
        & $68.1\pm0.4^{*}$
        & $63.9 \pm 2.7$
        & $51.3 \pm 1.9$
        & $80.4 \pm 0.4$
        \\

\hline 
\multirow{2}{*}{Unimodals}
& $\mathrm{I}:1.9\pm0.0$
& $\mathrm{S}:56.3\pm0.5$
        & $\mathrm{V}:55.6\pm3.4$
        & $\mathrm{V}:42.7\pm1.2$
        & $\mathrm{V}:65.2\pm0.2$
        \\
& $\mathrm{F}:87.7\pm0.5$
& $\mathrm{T}:68.0\pm0.4$
        & $\mathrm{A}:56.5\pm3.7$
        & $\mathrm{A}:28.8\pm1.7$
        & $\mathrm{T}:78.4\pm1.3$
        \\

Joint Training
& $1.9\pm0.1$
 & $66.8\pm1.4$
            & $62.6\pm5.8^{*}$
        & $47.7\pm1.5^{*}$
        & $80.5\pm0.2^{*}$
        \\

        MCR 
        & $2.0 \pm 0.18$
        & $68.6 \pm 0.8$
        & $76.1\pm1.6^{*}$
        & $55.2\pm1.8^{*}$
        & $\mathbf{80.8\pm0.4}^{*}$
        \\

        \midrule
        \METHOD \\
        \MaskShort{}
        & $\mathbf{1.6 \pm 0.1}$
        & $\mathbf{69.1 \pm 0.4}$
        & $74.7\pm 1.7$
        & $57.0 \pm 1.8$
        & $80.5 \pm 0.8$
        \\

        \ContrastShort{}
        & $2.6 \pm 0.2$
        & $68.7 \pm 0.5$
        & $\mathbf{76.2 \pm 2.0}$
        & $\mathbf{57.2 \pm 2.0}$
        & $80.3 \pm 0.6$
        \\

        \bottomrule
    \end{tabular}
\end{table*}

\begin{table*}[t]
    \begin{minipage}[t]{0.59\linewidth}
        \centering
        \caption{
        Test results on two modality benchmarks MOSI (MO), UR-FUNNY (UR), MUSTARD (MU), by unsupervised, supervised, and our methods.}
        \label{tab:realworld_unified_2}
        \begin{tabular}{@{}l@{\hspace{0.03\linewidth}}c@{\hspace{0.03\linewidth}}c@{\hspace{0.03\linewidth}}c@{}}
            \toprule
    
            Method
            & MO $\uparrow$
            & UR $\uparrow$
            & MU $\uparrow$
            \\
            \midrule
    
            FactorCL 
            & $51.2\pm1.6^{*}$
            & $60.5\pm0.8^{*}$
            & $55.8\pm0.9^{*}$
            \\

            CoMM 
            & $63.7\pm2.5^{*}$
            & $63.3\pm0.5^{*}$
            & $64.4\pm1.1^{*}$
            \\
    
            InfMasking 
            & $69.0\pm1.2^{*}$
            & $64.3\pm0.9^{*}$
            & $66.8\pm2.5^{*}$
            \\
    
    \hline 
    \multirow{2}{*}{Unimodals}
    & $\mathrm{V}:55.1\pm1.3$
    & $52.4\pm0.5$
    & $57.7\pm2.2$
            \\
    & $\mathrm{T}:72.0\pm1.4$
    & $61.7\pm0.7$
    & $63.4\pm2.0$
            \\
    
    Joint Training
    & $72.9\pm1.6$
    & $61.7\pm1.0$
    & $60.6\pm2.0$
            \\

            MCR 
            & $75.0 \pm 1.2$
            & $54.4 \pm 2.1$
            & $60.5 \pm 3.5$
            \\
    
            \midrule
            \METHOD \\
            \MaskShort{}{}
            & $\mathbf{75.8 \pm 0.8}$
            & $63.7 \pm2.0$
            & $\mathbf{68.6 \pm 5.5}$
            \\
    
            \ContrastShort{}
            & $74.3 \pm 1.9$
            & $\mathbf{65.0 \pm 1.8}$
            & $66.0 \pm 2.9$
            \\
    
            \bottomrule
        \end{tabular}
    \end{minipage}\hfill
    \begin{minipage}[t]{0.37\linewidth}
        \centering
        \caption{Architecture ablation on MOSI (MO), UR-FUNNY (UR), MUSTARD (MU) validation sets.}
        \label{tab:trifeature_ablation}
        \begin{tabular}{@{}l@{\hspace{0.04\linewidth}}c@{\hspace{0.03\linewidth}}c@{\hspace{0.03\linewidth}}c@{}}
            \toprule
            \METHOD{}
                & MO  $\uparrow$
                & UR $\uparrow$
                & MU  $\uparrow$ \\
            \midrule
            \MaskShort{}
                & 73.6 &  63.8 & 63.8    \\
            \quad w/o Flow
                &  68.7 & 59.6 & 48.6    \\
            \quad  w/o $z^c$
                & 64.6 &  62.0 &  57.3    \\
            \midrule
            \ContrastShort{}
                & 76.2 & 63.6 & 59.4     \\
            \quad  w/o Flow
                & 70.2 & 61.4 & 55.8    \\
            \quad  w/o $z^c$
                & 72.9 & 58.9 & 52.9     \\
            \bottomrule
        \end{tabular}
    \end{minipage}
\end{table*}

\begin{table}[t]
    \centering
    \caption{Test accuracies (\%) on three-modality benchmarks.}
    \label{tab:multibench_3mod_main}
    \begin{tabular}{lccc}
        \toprule
        Method
            & UR-FUNNY $\uparrow$
            & V\&T Contact $\uparrow$
            & MOSEI $\uparrow$ \\
        \midrule

        CoMM
            & $64.8\pm1.1^{*}$
            & $94.1\pm0.2^{*}$
            & $79.0 \pm 0.3$ \\

        InfMasking
            & $65.6\pm1.2^{*}$
            & $94.1\pm0.1^{*}$
            & $78.9 \pm 0.8$ \\

        MCR
            & $66.6 \pm 1.5$
            & $94.1 \pm 0.3$
            & $81.1\pm0.4^{*}$ \\

        \midrule

        \Mask{}
            & $66.6 \pm 1.8$
            & $94.0 \pm 0.3$
            & $80.7 \pm 0.7$ \\

        \Contrast{}
            & $66.0 \pm 0.9$
            & $93.6 \pm 0.9$
            & $81.4 \pm 0.5$ \\

        \bottomrule
    \end{tabular}%
\end{table}


\paragraph{Experiments with Two Input Modalities.}
Tables~\ref{tab:realworld_unified} and \ref{tab:realworld_unified_2} summarize the performance on the two-modality benchmarks. We observe that the optimal auxiliary objective varies depending on the dataset: \Mask{} excels on V\&T EE, MIMIC, MOSI, and MUSTARD, whereas \Contrast{} achieves the best performance on UR-FUNNY, CREMA-D, and UCF101. Across all tasks, \METHOD{} consistently outperforms previous unsupervised and supervised methods, indicating better utilization of all decomposed predictive information.
Optimal hyperparameters and sensitivity analysis are discussed in Appendix (cf. 
Table~\ref{tab:hyperparameter_config} 
and Table~\ref{tab:hyperparameter_tuning_full}). 

\paragraph{Ablation studies.} 
We conduct two sets of ablation studies to examine the contributions of the proposed model components and training objectives.
As shown in Table~\ref{tab:trifeature_ablation}, removing the invertible transformations (w/o Flow) or the separate nuisance component (w/o $z^c$) significantly degrades performance across all tasks. 
In the appendix, we ablate the loss terms in Table~\ref{tab:loss_weight_ablation}.
The results confirms that both auxiliary representation objective and likelihood are indispensable for effectively capturing predictive information, validating our structural design.

\paragraph{Experiments with Three Input Modalities.}
We extend our LVM with one more modality, and consider UR-FUNNY with visual, textual, and acoustic features; V\&T contact prediction with image, force, and proprioceptive observations; and MOSEI with visual, textual, and acoustic features.
As shown in Table~\ref{tab:multibench_3mod_main}, \METHOD{} consistently outperforms baselines, demonstrating that it can scale to multiple modalities while preserving strong predictive performance.

\section{Conclusion}
\label{sec:conclusion}

We introduced \METHOD{}, a structured supervised latent variable model that tackles modality competition in multimodal fusion. By disentangling representations into shared predictive, modality-specific predictive, and task-irrelevant nuisance factors, our framework isolates and better utilizes diverse multimodal evidence. Integrating this structured likelihood model with representation learning---via dimension-preserving flows---bridges the gap between classic statistical models and high-dimensional deep networks. Empirically, \METHOD{} mitigates unimodal bias, outperforming strong baselines across diverse real-world benchmarks, and scales seamlessly to more than two modalities. Ultimately, our approach offers a unified latent-variable lens for characterizing complex continuous multimodal interactions. 
Given the inherent complexity of synergy, cleanly extracting all PID components into completely separate latent variables remains an open question. Future work will explore extending our framework to handle partially missing modalities~\citep{wu2018multimodal,lee2021privateshared,lee2021variational}.

\subsection*{AI use statement}

Parts of this manuscript were drafted with assistance from AI.  All AI-assisted content was subsequently reviewed, verified, and edited by the authors.

\subsection*{Reproducibility statement}
We provide the information needed throughout the main paper and appendix. The main paper describes the proposed model, learning objectives, and common experimental protocol. The appendix provides additional details on dataset preprocessing and evaluation protocols, source-specific architectures, hyperparameter configurations, and experimental settings.

\bibliographystyle{iclr2027_conference}
\bibliography{multiview}

\clearpage
\appendix

\section{Appendix}
\label{sec:appendix}

\subsection{Datasets}
\label{sec:dataset}

We evaluate the proposed framework across affective computing, multimodal language understanding, robotics, and audio--visual recognition benchmarks.

\paragraph{MultiBench.}
\textbf{MOSI} contains $2{,}199$ opinion segments paired with visual, acoustic, and textual observations. In the InfMasking-style evaluation, we use the two-modality setting and evaluate the representation with binary sentiment classification. 

\textbf{MOSEI} extends this setting to approximately $23{,}000$ monologue clips and is evaluated under both the two-modality
supervised protocol inherited from MCR and the visual--acoustic--textual three-modality setting. 

\textbf{UR-FUNNY} contains $16{,}514$ examples from $1{,}866$ TED-talk videos and evaluates humor recognition. We use both its two-modality configuration and its visual--acoustic--textual configuration for the three-modality experiments. 

\textbf{MUSTARD} contains $690$ balanced
dialogue utterances and evaluates sarcasm recognition
\citep{liang2021multibench,wen2026infmasking,kontras2025balancing}.

\paragraph{Robotics benchmark.}
\textbf{Vision\&Touch} contains $150$ robotic manipulation trajectories with $1{,}000$ time steps per trajectory. We consider two tasks. End-effector prediction is a regression task evaluated from the two-modality representation,
whereas contact prediction is evaluated as a three-modality classification task using image, force, and proprioceptive observations. We report MSE$\times10^{-4}$ for end-effector prediction and classification accuracy for contact prediction.

\paragraph{Audio--visual benchmarks.}
\textbf{CREMA-D} evaluates six-way emotion recognition from synchronized face video and speech produced by $91$ actors. 

\textbf{UCF101} is used for audio--visual action recognition; following the MCR protocol, we retain the subset of action categories for which both modalities are available. These
benchmarks follow the supervised evaluation protocol used by MCR~\citep{kontras2025balancing}.

\textbf{MIMIC} combines a 24-hour sequence of clinical measurements with static patient descriptors and defines a binary diagnostic-group prediction task. It contains $53{,}423$ admissions from $38{,}597$ patients. We retain MIMIC in
the comparison where required to reproduce the benchmark coverage and averages reported by prior multimodal interaction methods.

\subsection{Architecture}
\label{sec:expts-architecture}

\begin{figure}[t]
    \centering
    \includegraphics[width=\linewidth]{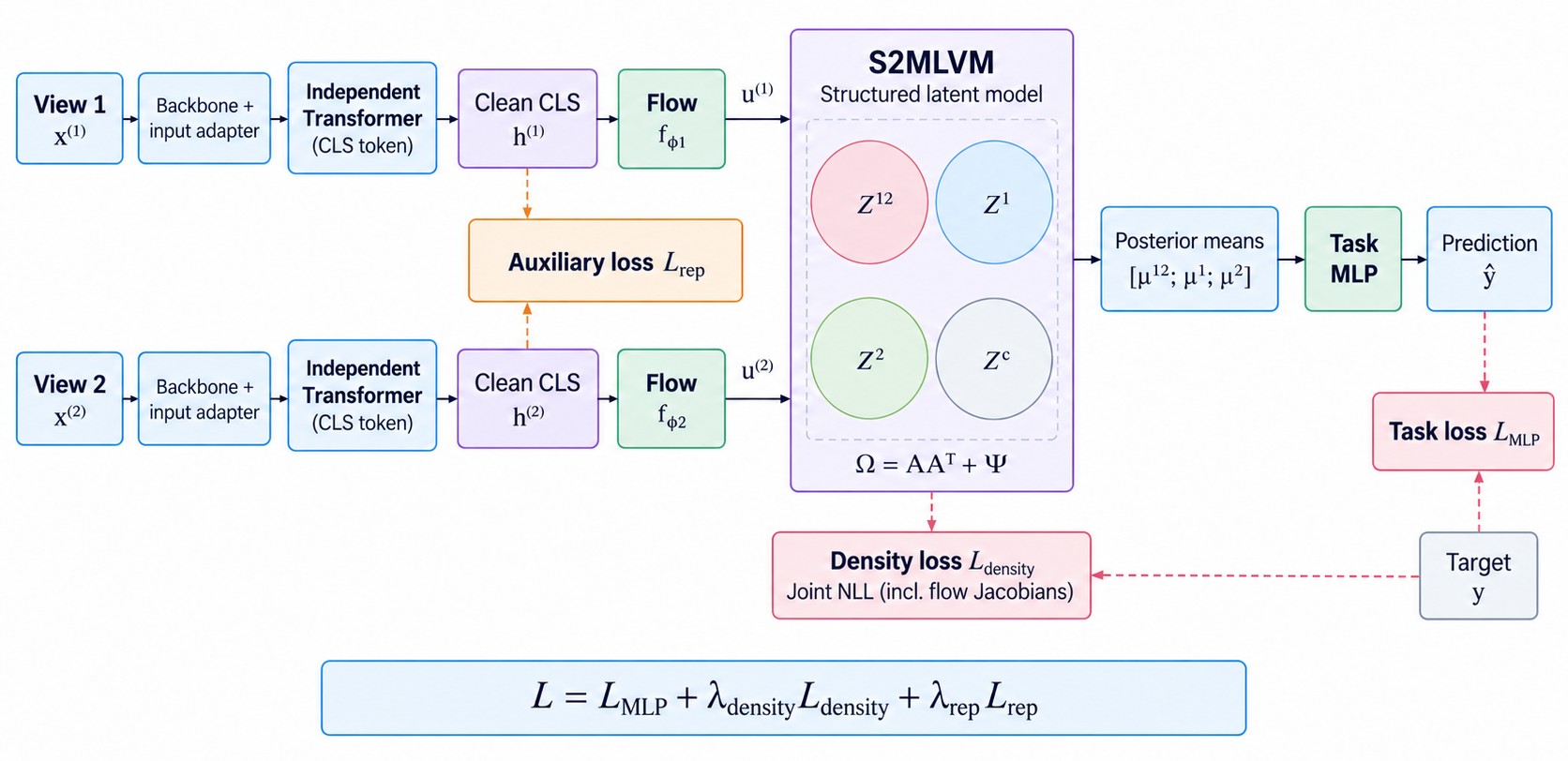}
    \caption{Overview architecture.}
    \label{fig:architecture}
\end{figure}
As illustrated in Figure~\ref{fig:architecture},the overall architecture consists of dataset-specific neural encoders, normalizing flows, and the \METHOD{} latent variable model. 
Each source pathway consists of a dataset-specific backbone and input adapter, followed by an independent Transformer with its own learnable CLS token. The clean, unprojected CLS representations are passed separately to their respective flows, without cross-modal attention or target inputs in the source encoders. 
We evaluate two model variants, \Mask{} and \Contrast{}, which share this identical architecture and differ only in their auxiliary representation loss. To make downstream predictions, the concatenated posterior means of the predictive latent blocks ($z^{12}$, $z^1$, and $z^2$) are extracted from the LVM and passed through a task MLP comprising two linear layers with an intervening ReLU activation, where the supervised task loss is applied.

However, specific details are handled differently for each dataset.
For MOSI, UR-FUNNY, MUSTARD, and MOSEI, we use separate temporal Transformers with modality-specific input projections to encode the precomputed visual and textual feature sequences.
For MIMIC, an MLP encodes the static patient descriptors, while a GRU encodes the multivariate clinical time series.
For Vision\&Touch end-effector regression, we use a ResNet-18 image backbone and a temporal force encoder as separate source pathways. The three-modality contact-prediction configuration additionally includes a dedicated proprioception encoder.
For CREMA-D and UCF101, we adopt the ResNet-18 backbone configuration of MCR~\citep{kontras2025balancing}, with independent visual and acoustic encoders operating on video frames and audio spectrograms, respectively.

Besides, the unimodal and joint-training baselines are reproduced following MCR~\citep{kontras2025balancing}. Unimodal training optimizes each modality independently with the supervised task loss, whereas joint training optimizes all modalities together using a single supervised task objective.

\subsection{Hyperparameter tuning}
\label{hyper-tuning}

The final comparison uses validation-selected hyperparameters, summarized in Table~\ref{tab:hyperparameter_config}. For \Mask{}, these are the mask ratio and batch size; for \Contrast{}, they are the InfoNCE temperature and batch size.
Here, mask ratio $m \in[0,1]$ denotes the fraction of input feature tokens randomly masked for the masked-reconstruction objective, temperature $\tau$ is the temperature parameter that scales the similarity logits before the softmax in the InfoNCE objective, and batch size $B$ denotes the number of training examples in each minibatch. The validation sensitivity of these hyperparameters is reported in Table~\ref{tab:hyperparameter_tuning_full}.

\begin{table}[t]
\centering
\caption{
Selected hyperparameter for each dataset.
}
\label{tab:hyperparameter_config}

\begin{tabular}{lcccccc}
    \toprule
    Dataset
    & \multicolumn{2}{c}{\Mask{}}
    & \multicolumn{2}{c}{\Contrast{}}
    & dim $k$
    & loss  \\
    \cmidrule(lr){2-3}
    \cmidrule(lr){4-5}
    & Mask ratio 
    & Batch size
    & Temperature
    & Batch size
    & 
    & $(\lambda_{\mathrm{density}}, \lambda_{\mathrm{rep}})$ \\
    \midrule

    MOSI
    & $0.55$
    & $96$
    & $0.10$
    & $32$
    & $4$ 
    & (1,1)\\

    MIMIC
    & $0.60$
    & $96$
    & $0.20$
    & $48$
    & $12$
    & (1,1)\\

    UR-FUNNY
    & $0.30$
    & $64$
    & $0.10$
    & $128$
    & $4$
    & (1,1)\\

    MUSTARD
    & $0.45$
    & $64$
    & $0.10$
    & $128$
    & $4$ 
    & (1,1)\\

    CREMA-D
    & $0.55$
    & $32$
    & $0.15$
    & $8$
    & $16$ 
    &(0.2,0.8) \\

    UCF101
    & $0.55$
    & $16$
    & $0.05$
    & $16$
    & $12$ 
    &(0.2,1) \\

    MOSEI
    & $0.85$
    & $64$
    & $0.05$
    & $32$
    & $12$ 
    &(0.5,0.8)\\

    \bottomrule
\end{tabular}
\end{table}

\begin{table*}[t]
    \centering
    \caption{
        Hyperparameter sensitivity on validation sets.
    }
    \label{tab:hyperparameter_tuning_full}
    \begin{tabular}{llcc}
        \toprule
        Study
        & Candidate
        & MOSI $\uparrow$
        & MUSTARD $\uparrow$ \\
        \midrule

        \multirow{8}{*}{\shortstack[l]{Mask ratio\\(\Mask{})}}
        & $m=0.40$ & 77.6 & 65.2 \\
        & $m=0.45$ & 79.1 & \textbf{70.3} \\
        & $m=0.50$ & 77.8 & 66.7 \\
        & $m=0.55$ & \textbf{80.0} & 70.0 \\
        & $m=0.60$ & 75.6 & 62.7 \\
        & $m=0.65$ & 77.9 & 65.2 \\
        & $m=0.70$ & 79.7 & 64.4 \\
        & $m=0.75$ & 74.4 & 63.1 \\

        \addlinespace[4pt]
        \midrule

        \multirow{4}{*}{\shortstack[l]{Temperature\\(\Contrast{})}}
        & $\tau=0.05$  & 77.7 & 64.9 \\
        & $\tau=0.10$  & \textbf{79.9} & \textbf{69.2} \\
        & $\tau=0.15$  & 79.0 & 68.7 \\
        & $\tau=0.20$  & 79.2 & 67.1 \\

        \addlinespace[4pt]
        \midrule

        \multirow{5}{*}{\shortstack[l]{Batch size\\(\Mask{})}}
        & $B=32$  & 75.5 & 68.0 \\
        & $B=48$  & 78.9 & 65.8 \\
        & $B=64$  & 79.1 & \textbf{70.2} \\
        & $B=96$  & \textbf{79.6} & 64.7 \\
        & $B=128$ & 77.0 & 64.6 \\

        \addlinespace[4pt]
        \midrule

        \multirow{5}{*}{\shortstack[l]{Batch size\\(\Contrast{})}}
        & $B=32$  & \textbf{82.7} & 64.9 \\
        & $B=48$  & 80.8 & 68.8 \\
        & $B=64$  & 79.0 & 67.7 \\
        & $B=96$  & 79.0 & 65.9 \\
        & $B=128$ & 78.6 & \textbf{72.2} \\

        \bottomrule
    \end{tabular}
\end{table*}

\subsection{Loss Ablation}
\label{loss-blation}
We use a strict one-factor-at-a-time validation procedure. We conduct an ablation study on the loss coefficients to examine the contribution of each loss term, with the results reported in Table~\ref{tab:loss_weight_ablation}.

\begin{table*}
\centering
\caption{Loss-coefficient ablations on validation datasets.}
\label{tab:loss_weight_ablation}
 \begin{tabular}{lcccc}
        \toprule
        Variant
            & MOSI  $\uparrow$
            & MUSTARD  $\uparrow$ \\
        \midrule
        \Mask{} $(\lambda_{\mathrm{density}}=1, \lambda_{\mathrm{rep}}=1)$
            & 79.0 & 68.6 \\
        $\quad$ w/o $L_{\mathrm{density}}$
            & 64.6 & 63.2 \\
        $\quad$ w/o $L_{\mathrm{rep}}$
            & 73.3  & 66.5 \\
        \midrule
        \Contrast{} $(\lambda_{\mathrm{density}}=1, \lambda_{\mathrm{rep}}=1)$
            & 79.9 & 66.4 \\
        $\quad$ w/o $L_{\mathrm{density}}$
            & 75.4& 64.2 \\
        $\quad$ w/o $L_{\mathrm{rep}}$
            & 77.5& 62.8 \\
        \bottomrule
    \end{tabular}
\end{table*}
\section{Connection to Partial Information Decomposition (PID)}
\label{sec:interpretation}

In this section, we analyze the theoretical origins of predictive synergy within our proposed latent variable model. While the previous sections established the practical benefits of explicitly disentangling predictive and nuisance factors for supervised fusion, we now formalize how this structure captures synergistic interactions between modalities. We utilize the Partial Information Decomposition (PID) framework~\citep{Williams2010Nonnegative,bertschinger2014quantifying} to isolate synergistic information, demonstrating how specific latent factors---such as shared nuisance variables and task-relevant signals---interact to provide predictive power that is only accessible when multiple modalities are observed jointly.

Specifically, PID decomposes the joint and marginal mutual informations into four non-negative components: the unique information provided by each individual view ($\text{Unique}_1$ and $\text{Unique}_2$), the redundant information shared by both views ($\text{Redundancy}$), and the synergistic information that arises only from their combination ($\text{Synergy}$). These components satisfy the classic PID system of equations~\citep{Williams2010Nonnegative}:
\begin{align*}
    I_P(U^1, U^2; Y) &= \text{Unique}_1 + \text{Unique}_2 + \text{Redundancy} + \text{Synergy}, \\
    I_P(U^1; Y) &= \text{Unique}_1 + \text{Redundancy}, \\
    I_P(U^2; Y) &= \text{Unique}_2 + \text{Redundancy}.
\end{align*}
Because this system is underdetermined, modern PID formulations define a constraint space $\Delta_P$ of \emph{adversarial joint distributions} $Q$ that preserve the pairwise source-target marginals of the true distribution $P$:
\begin{equation*}
    \Delta_P = \left\{ Q(U^1, U^2, Y) \;\middle|\; Q(U^1, Y) = P(U^1, Y), \ Q(U^2, Y) = P(U^2, Y) \right\}.
\end{equation*}
By minimizing the joint information over $Q \in \Delta_P$, the framework isolates the synergistic interactions from the baseline predictive information. Following \citet{bertschinger2014quantifying} and \citet{liang2024multimodal}, this adversarial optimization yields formal definitions for the unique and synergistic information:
\begin{align*}
    \text{Unique}_j &= \min_{Q \in \Delta_P} I_Q(U^j; Y \mid U^{-j}), \\
    \text{Synergy} &= I_P(U^1, U^2; Y) - \min_{Q \in \Delta_P} I_Q(U^1, U^2; Y), \\
    \text{Redundancy} &= \max_{Q \in \Delta_P} I_Q(U^1; U^2; Y) = I_P(U^1; Y) + I_P(U^2; Y) - \min_{Q \in \Delta_P} I_Q(U^1, U^2; Y).
\end{align*}

\subsection{The Adversarial Latent Framework}

To provide a self-contained analysis of synergy, we briefly recall the generative structure of our latent variable model. We assume the multi-view observations $U^1, U^2$ and the target $Y$ are generated from a set of mutually orthogonal Gaussian latent factors:\footnote{For notational simplicity in this section, we use uppercase $U^1, U^2, Y$ to denote the neural features $u^{(1)}, u^{(2)}$ and target embedding $\tilde{y}$ respectively, and we drop the parentheses on the block indices of the noise $\epsilon$.}
\begin{align*}
    U^1 &= \Lambda_1^{12} z^{12} + \Lambda_1^c z^c + \Lambda_1^1 z^1 + \epsilon^1 \\
    U^2 &= \Lambda_2^{12} z^{12} + \Lambda_2^c z^c + \Lambda_2^2 z^2 + \epsilon^2 \\
    Y &= B^{12} z^{12} + B^1 z^1 + B^2 z^2 + \epsilon^Y
\end{align*}
Here, $z^{12}$ is the task-relevant shared factor, $z^c$ is the task-irrelevant shared nuisance factor, $z^1, z^2$ are view-specific task-relevant private factors, and $\epsilon^j \sim \mathcal{N}(\mathbf{0}, \Sig_{\epsilon^j})$ is the full-rank independent private noise. We assume the nuisance loadings $\Lambda_j^c$ are orthogonal to the combined task-relevant loadings $\tilde{\Lambda}_j = [\Lambda_j^{12}, \Lambda_j^j]$ (i.e., $(\tilde{\Lambda}_j)^\top \Lambda_j^c = \mathbf{0}$), and the task-relevant loading matrices $\tilde{\Lambda}_j$ have full column rank.

For our Gaussian model, any adversarial distribution $Q \in \Delta_P$ must preserve the true marginal predictive covariances ($\Sig_{U^1 U^1}, \Sig_{U^2 U^2}, \Sig_{U^1 Y}, \Sig_{U^2 Y}, \Sig_{YY}$). Because these marginals are fixed, the only parameter the adversary can manipulate is the cross-source covariance $\Sig_{U^1 U^2}$. This allows us to frame any valid adversary $q \in \Delta_P$ as a pure noise-injection strategy: the adversary retains the true generative process $p$ (leaving the latent factors and marginal variances untouched), but injects a cross-covariance matrix $C$ directly between the independent private noises $\epsilon^1, \epsilon^2$. This creates the adversarial joint covariance:
\begin{equation*}
    \Sig_{UU}^{(Q)} = \Sig_{UU}^{(p)} + \begin{bmatrix} \mathbf{0} & C \\ C^\top & \mathbf{0} \end{bmatrix}.
\end{equation*}
Because the injected matrix $C$ alters only the off-diagonal cross-source covariance, any choice of $C$ that keeps the resulting $\Sig_{UU}^{(Q)}$ positive semi-definite guarantees $Q \in \Delta_P$. The question then becomes: what constitutes the `worst-case'' cross-covariance $C$? This depends
on the chosen adversarial objective.

\subsection{The Conditional Independence Adversary ($q^*_{CE}$)}

To formally capture the notion of decoupling the views, we first define the least-informative distribution using the Maximum Entropy (MaxEnt) principle applied to the sources. We seek the distribution $q^*_{CE}$ that maximizes the conditional entropy of the sources $h_Q(U \mid Y)$, which is equivalent to maximizing the determinant $|\Sig^{(Q)}_{U \mid Y}|$:
\begin{equation*}
    q^*_{CE} = \arg\max_{Q} |\Sig^{(Q)}_{U \mid Y}| \quad \text{subject to} \quad \Sig^{(Q)} \succeq 0.
\end{equation*}
By maximizing the residual uncertainty of the sources, the MaxEnt objective forces the observations to be as unstructured and independent as possible once the target $Y$ is known. This adversary implies that, conditioned on $Y$, we have independent $U^1$ and $U^2$.

\begin{lemma}[Least-Informative Coupling via Maximum Conditional Entropy]
\label{lem:nuisance-decoupling}
Suppose the true data distribution $p$ is generated by the model. The true model contains synergy because the shared nuisance factor $z^c$ allows the views to cooperatively suppress shared noise, yielding a cleaner prediction of $Y$. Furthermore, because the independent task-relevant factors ($z^{12}, z^1, z^2$) all independently cause $Y$, they form a classical V-structure. While the private factors $z^1$ and $z^2$ are unconditionally independent, conditioning on their shared effect $Y$ renders all task factors conditionally dependent via the ``explaining away'' effect. At the observation level, this structure induces a negative conditional correlation between the views, mathematically captured by $-\Sig_{U^1 Y} \Sig_{YY}^{-1} \Sig_{Y U^2}$.

To achieve conditional independence, the optimal adversary $q^*_{CE}$ must cancel both of these mechanisms. Under the noise-injection framework, this is achieved by injecting the negative of the true conditional cross-covariance into the private noise:
\begin{equation*}
    C_{CE} = -\Sig^{(p)}_{U^1 U^2 \mid Y}.
\end{equation*}
\end{lemma}

\begin{proof}
The objective is to maximize the determinant of the joint conditional covariance matrix $|\Sig^{(Q)}_{U \mid Y}|$. By the properties of multivariate Gaussians, $\Sig^{(Q)}_{U \mid Y}$ is given by the Schur complement of the target variance $\Sig_{YY}$. Since $Q$ is formed by injecting $C$ into the off-diagonal of the observations, the conditional covariance under $Q$ is the true conditional covariance plus the injected noise block:
\begin{equation*}
    \Sig^{(Q)}_{U \mid Y} = \Sig^{(Q)}_{UU} - \Sig_{UY} \Sig_{YY}^{-1} \Sig_{Y U} = \Sig^{(p)}_{U \mid Y} + \begin{bmatrix} \mathbf{0} & C \\ C^\top & \mathbf{0} \end{bmatrix} = \begin{bmatrix} \Sig_{U^1 \mid Y}^{(p)} & \Sig^{(p)}_{U^1 U^2 \mid Y} + C \\ \Sig^{(p)}_{U^2 U^1 \mid Y} + C^\top & \Sig_{U^2 \mid Y}^{(p)} \end{bmatrix}.
\end{equation*}
By Hadamard's inequality~\citep[Theorem 17.9.4]{CoverThomas06a}, the determinant of a block matrix with fixed diagonal blocks is maximized when the off-diagonal blocks are zero. Therefore, to maximize $|\Sig^{(Q)}_{U \mid Y}|$, the adversary must set $C_{CE} = -\Sig^{(p)}_{U^1 U^2 \mid Y}$. For jointly Gaussian variables, a zero conditional cross-covariance implies conditional independence, achieving $U^1 \perp\!\!\!\perp U^2 \mid Y$. Furthermore, because this choice makes $\Sig^{(Q)}_{U \mid Y}$ a block diagonal matrix whose diagonal blocks are the valid marginals from $p$, the conditional covariance $\Sig^{(Q)}_{U \mid Y}$ is positive semi-definite (PSD). By the properties of the Schur complement, a joint block covariance matrix is PSD if and only if its lower-right block ($\Sig_{YY}$) and its Schur complement ($\Sig^{(Q)}_{U \mid Y}$) are both PSD. Since the target variance $\Sig_{YY} \succ 0$ is fixed by the constraint space, and our choice of $C_{CE}$ ensures $\Sig^{(Q)}_{U \mid Y} \succeq 0$, the resulting full joint covariance matrix $\Sig^{(Q)}$ is guaranteed to be a valid PSD matrix. This confirms that $q^*_{CE}$ is a feasible point in the constraint space $\Delta_P$.
\end{proof}

\begin{remark}[The Three Mechanisms of Synergy: Suppression, Averaging, and Explaining-Away]
This algebraic solution reveals how $q^*_{CE}$ removes the sources of multi-view synergy. By substituting the generative components, the true conditional cross-covariance decomposes into three terms, each corresponding to a synergistic mechanism:
\begin{equation*}
    \Sig^{(p)}_{U^1 U^2 \mid Y} = \underbrace{\Lambda_1^c (\Lambda_2^c)^\top}_{\text{1. Shared Nuisance}} + \underbrace{\Lambda_1^{12} (\Lambda_2^{12})^\top}_{\text{2. True Shared Signal}} - \underbrace{\Sig_{U^1 Y} \Sig_{YY}^{-1} \Sig_{Y U^2}}_{\text{3. Induced V-Structure Effect}}.
\end{equation*}
By injecting $C_{CE} = - \Sig^{(p)}_{U^1 U^2 \mid Y}$ into the noise, the adversary acts as a filter that cancels all three mechanisms:
\begin{enumerate}
    \item \textbf{Suppression (Driven by $z^c$):} When the views share a nuisance factor $z^c$, the optimal joint regression linearly cross-references the views to subtract and cancel out the shared interference. 
    Consider a toy system where $U^1 = Y + z^c$ and $U^2 = z^c$. Marginally, $U^2$ is uninformative about $Y$. However, a joint regression model computing the optimal linear estimator ($W = \Sig_{UU}^{-1} \Sig_{UY}$) assigns a negative weight to $U^2$. This allows the model to subtract the shared noise, computing $U^1 - U^2 = Y$ and recovering the target. $q^*_{CE}$ eliminates this mechanism by injecting $-\Lambda_1^c (\Lambda_2^c)^\top$ to decouple $z^c$ across the views, turning it into independent private noise so it can no longer be subtracted.
    \item \textbf{Averaging (Driven by $z^{12}$):} Because both views observe the true shared signal $z^{12}$ corrupted by independent private noise, they provide redundant measurements. For example, if $U^1 = Y + \epsilon^1$ and $U^2 = Y + \epsilon^2$ (with $\epsilon^1 \perp\!\!\!\perp \epsilon^2$), the optimal joint estimator applies positive weights to both views to compute their average. This averaging boosts the signal-to-noise ratio by suppressing the uncorrelated private noises, yielding a cleaner prediction of $Y$ than either view could individually. $q^*_{CE}$ breaks this capability by injecting $-\Lambda_1^{12} (\Lambda_2^{12})^\top$, which cancels the true shared signal correlation across the views, destroying the redundancy that enabled the averaging.
    \item \textbf{V-Structure Exploitation (Driven by $z^{12}, z^1, z^2$):} Because the independent task-relevant factors all independently cause the target $Y$, they form a classical joint V-structure. Observing $Y$ renders all task factors conditionally dependent via the ``explaining away'' effect, which a joint model exploits. $q^*_{CE}$ breaks this cooperative dependence because its injected noise ($+\Sig_{U^1 Y} \Sig_{YY}^{-1} \Sig_{Y U^2}$) cancels the induced negative explaining-away correlation.
\end{enumerate}
Together, these mechanisms dictate that any latent structure---whether shared task-relevant, shared nuisance, or conditionally induced by a V-structure---provides an opportunity for synergistic predictive power. The baseline distribution $q^*_{CE}$ eliminates them all.
\end{remark}

\subsection{The Minimum Mutual Information Adversary ($q^*_{MI}$)}

The original PID framework~\citep{bertschinger2014quantifying} defines its adversary $q^*_{MI}$ through an alternative objective: to find the distribution that minimizes the joint mutual information $I_Q(Y; U)$. Because $H(Y)$ is fixed across $\Delta_P$, this is equivalent to maximizing the residual generalized variance $|\Sig^{(Q)}_{Y \mid U}|$:
\begin{equation*}
    q^*_{MI} = \arg\max_{Q \in \Delta_P} |\Sig^{(Q)}_{Y \mid U}|.
\end{equation*}

The connection between the two adversarial objectives is governed by the determinant identity:
\begin{equation*}
    |\Sig^{(Q)}_{Y \mid U}| = |\Sig_{YY}| \frac{|\Sig^{(Q)}_{U \mid Y}|}{|\Sig^{(Q)}_{UU}|}.
\end{equation*}
This identity follows directly from factoring the determinant of the joint covariance matrix $\Sig^{(Q)}$ in two equivalent ways via its Schur complements: $|\Sig^{(Q)}| = |\Sig_{YY}| |\Sig^{(Q)}_{U \mid Y}| = |\Sig^{(Q)}_{UU}| |\Sig^{(Q)}_{Y \mid U}|$.
To minimize the mutual information, $q^*_{MI}$ manipulates the cross-covariance matrix $C$ to reduce the denominator $|\Sig^{(Q)}_{UU}|$. Mathematically, injecting correlation reduces the generalized variance of a joint distribution. Because the marginal blocks $\Sig_{U^1 U^1}$ and $\Sig_{U^2 U^2}$ are fixed by $\Delta_P$, the joint determinant factors via the Schur complement as $|\Sig^{(Q)}_{UU}| = |\Sig_{U^1 U^1}| |\Sig_{U^2 U^2} - \Sig_{U^2 U^1}^{(Q)} \Sig_{U^1 U^1}^{-1} \Sig_{U^1 U^2}^{(Q)}|$. Injecting a cross-covariance subtracts a positive semi-definite matrix from the second term, thereby reducing the overall determinant. $q^*_{MI}$ exploits this property by injecting positive correlation between the private noises across the views. This compresses the joint noise distribution to align with the signal, maximizing the residual target variance.

Like $q^*_{CE}$, $q^*_{MI}$ must first decouple the response-irrelevant nuisance factor $z^c$. If $z^c$ remains shared across the views, the optimal estimator will exploit this correlation to subtract the views and cancel the nuisance interference (the suppression mechanism). To prevent this cooperative cancellation and minimize the mutual information $I_Q(Y; U)$, $q^*_{MI}$ adopts the same strategy as $q^*_{CE}$ within the nuisance subspace: it injects negative correlation to decouple $z^c$, turning it into independent private noise.

However, the two adversaries diverge fundamentally in how they treat the task-relevant subspace. The optimization objective of $q^*_{CE}$ yields a clean, term-by-term matrix decomposition that explicitly decouples the task variables to achieve conditional independence ($I_{q^*_{CE}}(U^1; U^2 \mid Y) = 0$). In contrast, the objective of $q^*_{MI}$ mathematically couples these variables together through the inverse covariance matrix. Because this coupling prevents a simple discrete factorization for the task subspace, the adversary must instead holistically inject positive correlation to align the joint noise ellipse with the combined signal direction. While this alignment effectively hinders the regression model, it inevitably introduces a conditional dependence between the views. Thus, the optimal mutual information adversary sacrifices conditional independence ($I_{q^*_{MI}}(U^1; U^2 \mid Y) > 0$) to achieve its objective.

To see this mechanically, consider a simplified scalar system where the private task-relevant factors $z^1, z^2$ are absent. The target is defined purely by the shared factor $Y = z^{12}$, with $z^{12} \sim \mathcal{N}(0, 1)$. The observations are $U = \Lambda z^{12} + n$, where the stacked vector $\Lambda = \begin{bmatrix} \lambda_1 \\ \lambda_2 \end{bmatrix}$ defines the specific 1-dimensional ``signal direction''. Here, $n \sim \mathcal{N}(\mathbf{0}, \Sig_N^{(p)})$ represents the sum of the full-rank private noise and the completely decoupled nuisance factors. Under the noise-injection framework, the adversary injects its cross-covariance $c$ exclusively into this noise block, resulting in the total adversarial noise covariance $\Sig_N^{(Q)} = \begin{bmatrix} v_1 & c \\ c & v_2 \end{bmatrix}$.

The joint mutual information is minimized when $|\Sig^{(Q)}_{Y \mid U}|$ is maximized. Under our explicit model, $\Sig_{YY} = 1$, $\Sig_{YU} = \Lambda^\top$, and $\Sig_{UU}^{(Q)} = \Lambda \Lambda^\top + \Sig_N^{(Q)}$. Substituting these into the residual variance yields:
\begin{equation*}
    |\Sig^{(Q)}_{Y \mid U}| = 1 - \Lambda^\top \left(\Lambda \Lambda^\top + \Sig_N^{(Q)}\right)^{-1} \Lambda.
\end{equation*}
To simplify the inverted term, we apply the Woodbury matrix identity, $(\mathbf{A} + \mathbf{U}\mathbf{C}\mathbf{V})^{-1} = \mathbf{A}^{-1} - \mathbf{A}^{-1}\mathbf{U}(\mathbf{C}^{-1} + \mathbf{V}\mathbf{A}^{-1}\mathbf{U})^{-1}\mathbf{V}\mathbf{A}^{-1}$, setting $\mathbf{A} = \Sig_N^{(Q)}$, $\mathbf{U} = \Lambda$, $\mathbf{V} = \Lambda^\top$, and $\mathbf{C} = 1$. Letting $K(c) = \Lambda^\top (\Sig_N^{(Q)})^{-1} \Lambda$ represent the Signal-to-Noise Ratio (SNR) penalty, the inverse becomes:
\begin{equation*}
    \left(\Sig_N^{(Q)} + \Lambda \Lambda^\top\right)^{-1} = (\Sig_N^{(Q)})^{-1} - (\Sig_N^{(Q)})^{-1} \Lambda \left(1 + K(c)\right)^{-1} \Lambda^\top (\Sig_N^{(Q)})^{-1}.
\end{equation*}
Substituting this expanded inverse back into the residual variance equation yields a step-by-step algebraic reduction:
\begin{align*}
    |\Sig^{(Q)}_{Y \mid U}| &= 1 - \Lambda^\top \left[ (\Sig_N^{(Q)})^{-1} - (\Sig_N^{(Q)})^{-1} \Lambda \left(1 + K(c)\right)^{-1} \Lambda^\top (\Sig_N^{(Q)})^{-1} \right] \Lambda \\
    &= 1 - \underbrace{\Lambda^\top (\Sig_N^{(Q)})^{-1} \Lambda}_{K(c)} + \underbrace{\Lambda^\top (\Sig_N^{(Q)})^{-1} \Lambda}_{K(c)} \left(1 + K(c)\right)^{-1} \underbrace{\Lambda^\top (\Sig_N^{(Q)})^{-1} \Lambda}_{K(c)} \\
    &= 1 - K(c) + \frac{K(c)^2}{1 + K(c)} \\
    &= \frac{1}{1 + K(c)}.
\end{align*}
Thus, maximizing the residual variance is exactly equivalent to minimizing the SNR penalty term $K(c)$:
\begin{equation*}
    K(c) = \begin{bmatrix} \lambda_1 & \lambda_2 \end{bmatrix} \begin{bmatrix} v_1 & c \\ c & v_2 \end{bmatrix}^{-1} \begin{bmatrix} \lambda_1 \\ \lambda_2 \end{bmatrix} = \frac{\lambda_1^2 v_2 + \lambda_2^2 v_1 - 2 c \lambda_1 \lambda_2}{v_1 v_2 - c^2}.
\end{equation*}
By taking the derivative with respect to $c$ and setting it to zero, we find the adversary optimally minimizes the SNR penalty at the roots $c = \lambda_1 v_2 / \lambda_2$ and $c = \lambda_2 v_1 / \lambda_1$. The true optimal adversarial correlation $c^*$ is the unique root that satisfies the positive semi-definite covariance constraint $c^2 < v_1 v_2$. Because $v_1, v_2 > 0$, both roots strictly match the sign of the signal correlation $\lambda_1 \lambda_2$, ensuring $c^*$ always matches this sign as well. This sign-matching mathematically guarantees that the adversarial noise is geometrically aligned with the task signal. For instance, if the views have positively aligned signal loadings ($\lambda_1 \lambda_2 > 0$), an optimal estimator would average the views to boost the signal while canceling independent noise. By injecting a positive noise correlation ($c^* > 0$), the adversary ensures that the noise components constructively reinforce each other upon addition. Conversely, if the signal extraction requires subtraction, the adversary injects negative correlation to ensure the subtracted noises similarly reinforce. Consequently, any linear combination that attempts to boost the signal will simultaneously amplify the geometrically aligned adversarial noise, structurally preventing the model from achieving synergistic noise cancellation.

\begin{remark}[The Mechanics of Adversarial Correlation]
The noise-injection framework captures the distinction between the adversaries: while both counteract cooperative nuisance suppression (i.e., injecting negative correlation to render the nuisance factor $z^c$ independent), their treatment of the task-relevant predictive subspace differs. 

To achieve conditional independence, $q^*_{CE}$ acts as a decoupling filter: it injects the negative conditional cross-covariance $C_{CE} = -\Sig^{(p)}_{U^1 U^2 \mid Y}$. In our scalar example, this corresponds to setting $c = 0$, making the noise conditionally independent so the regression model could average the views to extract a cleaner signal. 

In contrast, $q^*_{MI}$ goes further to minimize mutual information. Rather than just canceling the conditional correlation, it injects a positive correlation $C_{MI}$. By stretching and rotating the Gaussian noise ellipse $\Sig_N^{(Q)}$ so that it mimics the signal direction $\Lambda$, the adversary ensures that any linear estimator weight vector $w$ attempting to extract the signal (maximizing the projection $w^\top \Lambda$) simultaneously captures adversarial noise variance ($w^\top \Sig_N^{(Q)} w$). Mathematically, this minimizes the maximum achievable signal-to-noise ratio ($K(c) = \max_w \frac{(w^\top \Lambda)^2}{w^\top \Sig_N^{(Q)} w}$). Because the noise aligns with the signal, the estimator cannot project away the noise without also annihilating the signal itself, creating a structural dependency.
\end{remark}

\subsection{Shared Properties and Analytic Bounds}

Despite their differing mathematical objectives, both $q^*_{CE}$ and $q^*_{MI}$ are constrained within $\Delta_P$. Because of this shared constraint space, they share structural properties that define the PID decomposition.

\begin{remark}[The Marginal Preservation]
Because any valid adversary $Q \in \Delta_P$ matches the true predictor-target marginals ($I_Q(U^j; Y) = I_P(U^j; Y)$), the views maintain their individual marginal predictive relationships with $Y$ even when the shared task factor $z^{12}$ is adversarially corrupted. 
By the chain rule, $I_Q(U^1, U^2; Y) = I_P(U^{-j}; Y) + I_Q(U^j; Y \mid U^{-j})$. 
Since the marginal term is fixed, minimizing the joint information is identical to minimizing the conditional mutual information. Thus, because $q^*_{CE}$ is a sub-optimal minimizer of joint information ($I_{q^*_{CE}} \ge I_{q^*_{MI}}$), its resulting conditional mutual information is larger, acting as an upper bound on the true unique information.
\end{remark}

\begin{remark}[Analytic Bounds via the Tractable Proxy]
Because $q^*_{CE}$ provides a closed-form but sub-optimal minimization of the joint information ($I_{q^*_{MI}} \le I_{q^*_{CE}}$), substituting it into the PID equations yields analytic bounds on the true optimal quantities. Specifically, it underestimates Synergy ($S_{CE} \le S_{MI}$) and Redundancy ($R_{CE} \le R_{MI}$), while overestimating Unique Information ($U_j^{CE} \ge U_j^{MI}$).

Finally, this formulation reveals that conditional mutual information is distinct from unique information. By the chain rule, $I_P(U^1; Y \mid U^2) = \text{Unique}_1 + \text{Synergy}$. By removing the synergy, the adversaries reduce the conditional mutual information to its unique information component.
\end{remark}

\subsection{Bridge to Multi-View Learning}

This framework provides a theoretical bridge to classic multi-view learning. Classic multi-view algorithms typically invoke two distinct premises that are often bundled together: the \emph{redundancy assumption} (that the views provide identical predictive information, meaning zero unique information), which is relied upon by modern contrastive and bottleneck methods~\citep{federici2020learning, tosh2021contrastive}, and the \emph{conditional independence assumption} (that $U^1 \perp\!\!\!\perp U^2 \mid Y$
), which is foundational to classic algorithms like Co-training~\citep{blum1998combining} and CCA dimension reduction~\citep{Chaudh_09a}. Our baseline distribution $q^*_{CE}$ is the mathematical embodiment of the latter. Conditional independence does \emph{not} imply redundancy. Because $q^*_{CE}$ preserves the true marginals, it allows for amounts of unique predictive information; it merely forbids the views from interacting synergistically. By measuring our tractable proxy $S_{CE}$, we approximate the true synergy and directly quantify how much the true data violates the conditional independence assumption. When $S_{CE} \gg 0$, the views cooperate to cancel noise, indicating that algorithms assuming conditional independence will be suboptimal compared to joint models that can exploit synergistic cancellation.

\end{document}